\documentclass[letterpaper, 10 pt, conference]{ieeeconf} 
\usepackage[font=footnotesize]{subcaption}
\usepackage{epsfig} 
\usepackage{bm}
\usepackage{amssymb}
\usepackage{amsmath}
\newtheorem{proposition}{Proposition}
 	
\newtheorem{definition}{Definition} 
\usepackage{xcolor}

\IEEEoverridecommandlockouts                              

\title{\LARGE \bf
Rigid Body Motion Prediction with Planar Non-convex Contact Patch
}

\author{Jiayin Xie$^{1}$ and Nilanjan Chakraborty$^{2}$
\thanks{$^{1}$Jiayin Xie is with the Department of Mechanical Engineering,
        Stony Brook University, 100 Nicolls Rd, Stony Brook, NY 11794
        {\tt\small jiayin.xie@stonybrook.edu}}%
\thanks{$^{2}$Nilanjan Chakraborty is with the Department of Mechanical Engineering, Stony Brook University, 100 Nicolls Rd, Stony Brook, NY 11794
        {\tt\small nilanjan.chakraborty@stonybrook.edu}}%
}

\begin{document}

\maketitle
\thispagestyle{empty}
\pagestyle{empty}

\begin{abstract}
We present a principled method for motion prediction via dynamic simulation for rigid bodies in intermittent contact with each other where the contact is assumed to be a planar non-convex contact patch. The planar non-convex contact patch can either be a topologically connected set or disconnected set. Such algorithms are useful in planning and control for robotic manipulation. Most work in rigid body dynamic simulation assume that the contact between objects is a point contact, which may not be valid in many applications. In this paper,  by using the convex hull of the contact patch, we build on our recent work on simulating rigid bodies with convex contact patches, for simulating motion of objects with planar non-convex contact patches. We formulate a discrete-time mixed complementarity problem where we solve the contact detection and integration of the equations of motion simultaneously. Thus, our method is a geometrically-implicit method and we prove that in our formulation, there is no artificial penetration between the contacting rigid bodies. We solve for the equivalent contact point (ECP) and contact impulse of each contact patch simultaneously along with the state, i.e., configuration and velocity of the objects. We provide empirical evidence to show that our method can seamlessly capture transition between different contact modes like patch contact to multiple or single point contact during simulation.

\end{abstract}


\section{INTRODUCTION}
Rigid body motion prediction via dynamic simulation is a key enabling technology in solving robotic manipulation with multi-fingered hands, vibratory plates, and parts feeder design~\cite{ReznikC98,SongTVP04,VoseUL09,BerardNAT10}.
Many robotic manipulation tasks, e.g., extrinsic manipulation, manipulation by vibrating plates, involve point and surface contacts between the rigid body that is being manipulated and a flat plane on which the body rests~\cite{ReznikC98,VoseUL09,DafleR+14}. Furthermore, the occurrence of multiple intermittent contacts makes the prediction of the motion more complicated. There are applications in which the contact between two objects may be over a patch that can be modeled as a non-convex set. For example, Figure~\ref{figure_Motivation} shows a robot manipulator manipulating a T-shaped bar where the contact between the ground and the bar is a planar non-convex set. Such situations may arise when a robot manipulator with a parallel jaw gripper is trying to reconfigure a heavy bar with support from the table, so that it does not have to support the full weight. State-of-the-art dynamic simulation algorithms that can be used to predict motions during planning, usually assume point contact between two objects (except~\cite{XieC16, XieC18a}), which is clearly violated in Figure~\ref{figure_Motivation}. There are no well-principled approaches to predict the effect of applying a force/torque on the bar. In this paper, we seek to develop principled algorithms for predicting motion of rigid bodies in intermittent contact (via dynamic simulation), where the contacts can be modeled as a planar non-convex set.

\begin{figure}[htb]
\includegraphics[width=0.45\textwidth]{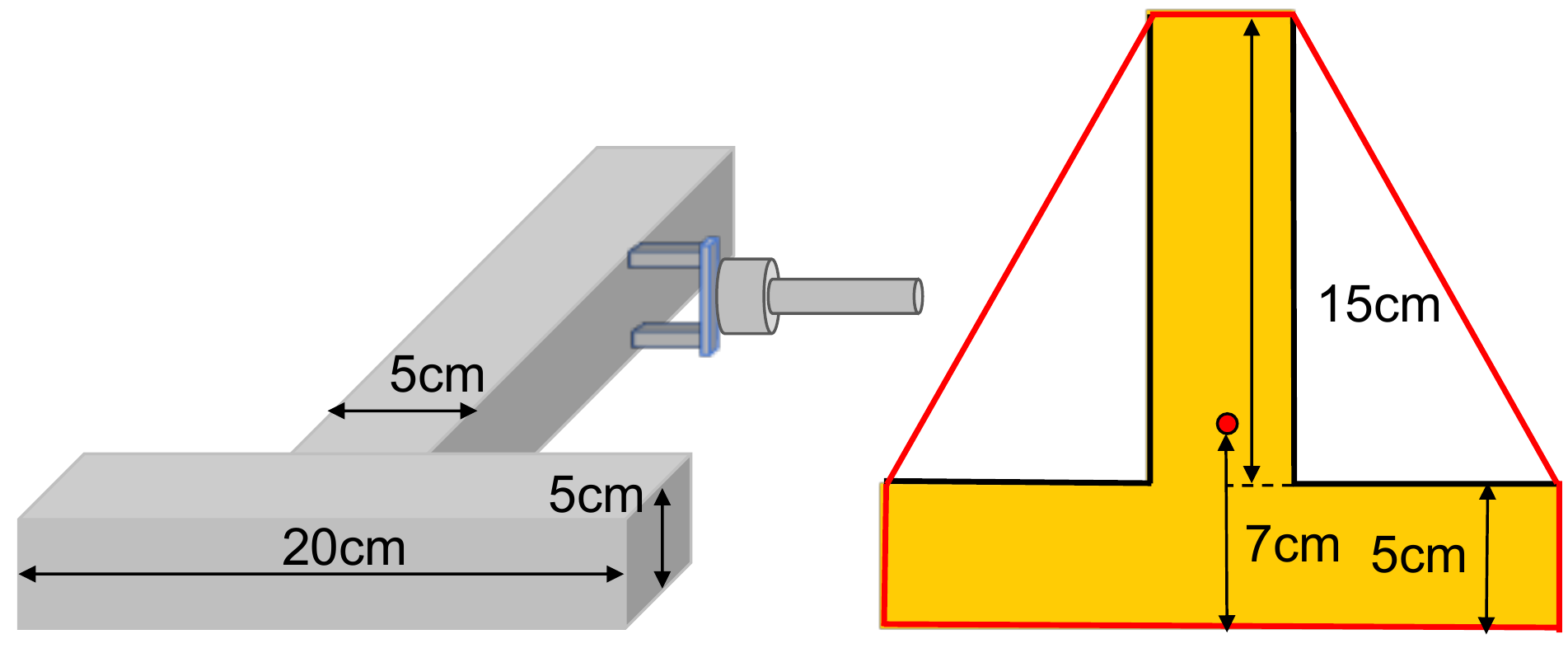}
\caption{(Left) A T-shaped bar on planar surface is manipulated by a gripper while being supported on the plane, (Right) where the planar contact between T bar and support is a non-convex T-shaped patch. The red line shows the convex hull for the contact patch.}
\label{figure_Motivation} 
\end{figure}

Figure~\ref{figure:different_contact_case} shows the key types of contact between objects. Most existing mathematical models for motion of objects with intermittent contact like Differential Algebraic Equation (DAE) models~\cite{Haug1986} and Differential Complementarity Problem (DCP) models~\cite{Cottle2009,Trinkle1997,PfeifferG08} assume the contact between the two objects is a single point contact (top left in Figure~\ref{figure:different_contact_case}). However, for convex contact patch (middle row in Figure~\ref{figure:different_contact_case}), the point contact assumption is not valid. In such case, multiple contacts point are usually chosen in an ad hoc manner, which can lead to inaccuracies in simulation (Please see~\cite{XieC16} for example scenarios). Recently, we developed an approach~\cite{XieC16} to simulate contacting rigid bodies with convex contact patches (line and surface contact). In~\cite{XieC18a}, we develop an approach for simulating contacting bodies where the contact patch is non-convex but can be modeled as a union of convex sets (bottom row, right column in Figure~\ref{figure:different_contact_case}).  In this paper, we focus on simulating bodies with planar non-convex contact patch, where the non-convex contact patch may not be a union of convex sets. The contact can be multiple point contacts or a general planar non-convex patch contact (top row, right column and bottom row in Figure~\ref{figure:different_contact_case}). Such situations arise when a robot is manipulating objects supported by a horizontal plane.
\begin{figure}[h]
\centering
\includegraphics[width=0.4\textwidth]{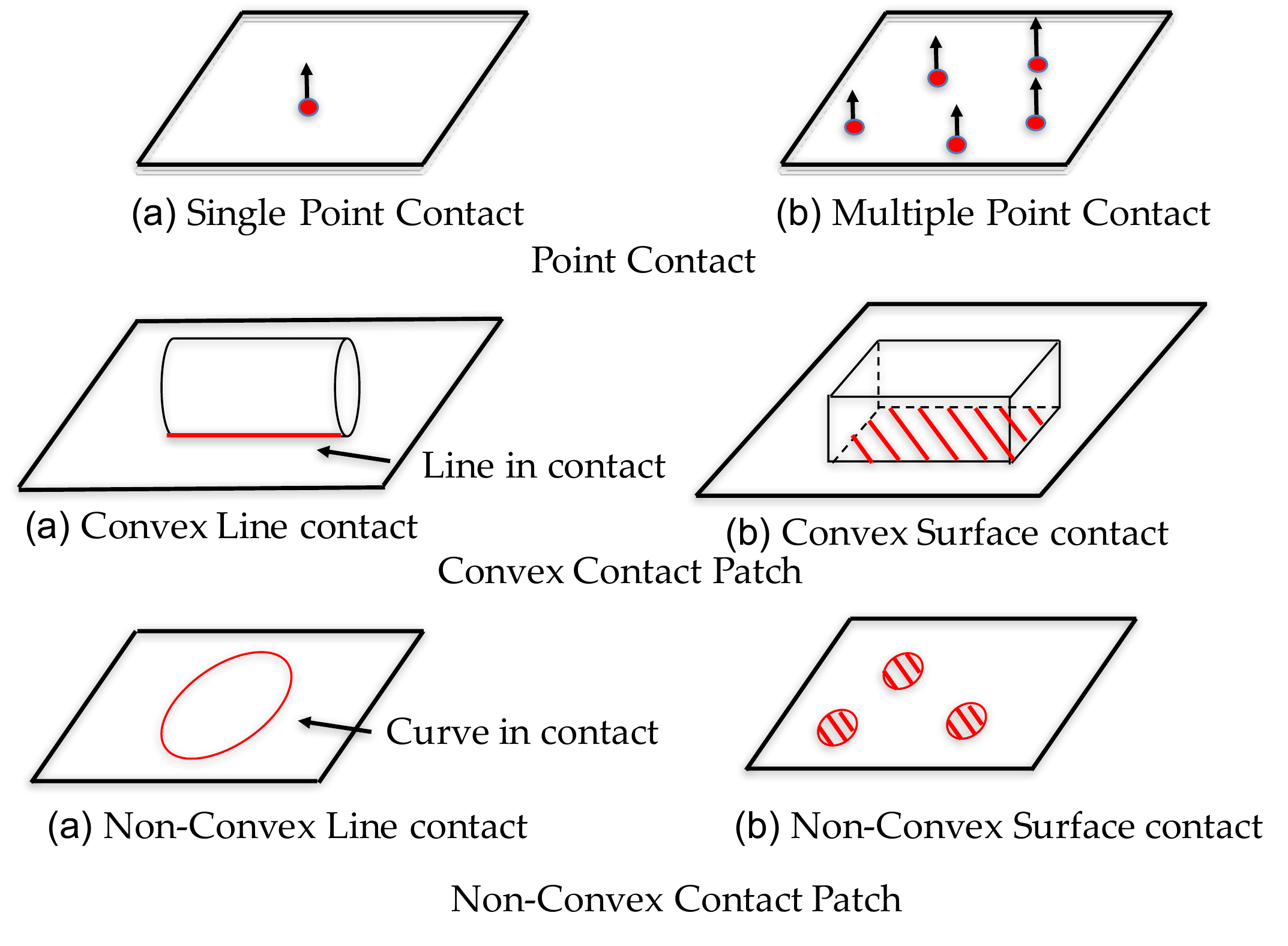}
\caption{Different types of contact between one object with a flat surface. Our focus in this paper is on simulating rigid bodies with type of contact shown in last row and first row, pane (b).}
\label{figure:different_contact_case} 
\end{figure} 
For a single convex contact patch, we know that there exists a unique point on the contact surface where the integral of total moment due to normal force acting on this point is zero. This point is used to model line or surface contact as a point contact and thus it is called the {\em equivalent contact point} (ECP)~\cite{XieC16}. Using the concept of ECP, in~\cite{XieC16}, we present a principled method for simulating intermittent contact with convex contact patches (line and surface contact). This method solves for the ECP as well as the contact impulses by {\em incorporating  the collision detection within the dynamic simulation time step}. This method is called the {\em geometrically implicit time-stepping method} because the geometric information of contact points and contact normal are solved as a part of the numerical integration procedure. In~\cite{XieC18a}, for non-convex contact patches that can be modeled as a union of convex sets, we use an ECP to model the effect of each convex contact patch and solve for the ECP and its associated contact wrenches on each contact patch separately. However, the limitation of this method was that the force/moment distribution and the ECP was non-unique, although the state of the object was unique. Furthermore, if there are more than three convex sets forming the non-convex patch, the force/moment in some of the contact patches may become zero.

In this paper, we extend the method in~\cite{XieC16}, by using the convex hull of the contact patch for modeling the contact constraints in the equations of motion. Although, we have intermittent contact and the contact patch may change (even topologically, we can go from a connected non-convex patch to multiple point contact), we do not need to form the convex hull of the contact patch during the simulation depending on the contact mode. Instead, we use the convex hull of the non-convex object that is being manipulated. And since we solve the collision detection problem simultaneously with the equations of motion (i.e., our method is geometrically implicit), we can ensure that the {\em convex hull of the contact patch} will always be automatically obtained through our contact detection constraints. Note that distinct from~\cite{XieC16}, the ECP may not be a point within the physical contact region (but it will be a point within the convex hull of the contact region). We prove that even though we are modeling a non-convex contact patch with an equivalent contact point that may not lie within the patch, the contact constraints are always satisfied at the end of the time-step and there is no artificial penetration between the objects. We show simulation results validating our approach with our previous models~\cite{XieC18a,XieC18b}. We also present simulation results showing that the object can seamlessly transition among different contact modes like non-convex patch contact, multiple point contact, line contact, and single point contact. 

\section{RELATED WORK}
In this section, we present the related work in rigid body dynamic simulation with a focus on methods for dealing with intermittent contact. There is also a substantial body of work on development of discretization schemes for integrating and simulating rigid body motion that we do not discuss here (please see the literature on variational integrators~\cite{MarsdenJM01,JohnsonRM09,KobilarovCD09} and references therein). 
We model the continuous time dynamics of rigid bodies that are in intermittent contact with each other as a Differential Complementarity Problem (DCP).  
Let $\bm{u}\in \mathbb{R}^{n_1}$,  $\bm{v}\in \mathbb{R}^{n_2}$ and let $\bm{g}$ :$ \mathbb{R}^{n_1}\times \mathbb{R}^{n_2} \rightarrow \mathbb{R}^{n_1} $, $\bm{f}$ : $ \mathbb{R}^{n_1}\times \mathbb{R}^{n_2} \rightarrow \mathbb{R}^{n_2}$ be two vector functions and the notation $0 \le \bm{x} \perp \bm{y} \ge 0$ imply that $\bm{x}$ is orthogonal to $\bm{y}$ and each component of the vectors is non-negative. 
\begin{definition}
The differential (or dynamic) complementarity problem~\cite{Facchinei2007} is to find $\bm{u}$ and $\bm{v}$ satisfying
$$\dot{\bm{u}} = \bm{g}(\bm{u},\bm{v}), \ \ \ 0\le \bm{v} \perp \bm{f}(\bm{u},\bm{v}) \ge 0 $$
\end{definition}
\begin{definition}
The mixed complementarity problem is to find $u$ and $v$ satisfying
$$\bm{g}(\bm{u},\bm{v})=0, \ \ \ 0\le \bm{v} \perp \bm{f}(\bm{u},\bm{v}) \ge 0.$$
If the functions $\bm{f}$ and $\bm{g}$ are linear, the problem is called a mixed linear complementarity problem (MLCP), otherwise, the problem is called a mixed nonlinear complementarity problem (MNCP). Our continuous time dynamics model is a DCP whereas our discrete-time dynamics model is a MNCP. 
\end{definition}

The DCP model formulates the intermittent contact between bodies in motion as a complementarity constraint~\cite{Lotstedt82, AnitescuCP96, Pang1996, StewartT96, Liu2005, DrumwrightS12, Todorov14}. 
DCP models are solved numerically with time-stepping schemes. The time-stepping problem is: {\em given the state of the system and applied forces, compute an approximation of the system one time step into the future. } Solving this problem repeatedly will give an approximate solution to the equations of motion.

There are different assumptions for forming the discrete equations of motion, which makes the system Mixed Linear Complementarity problem (MLCP)~\cite{AnitescuP97, AnitescuP02} or mixed non-linear complementarity problem (MNCP)~\cite{Tzitzouris01,NilanjanChakraborty2007}. The MLCP problem linearizes the friction cone constraints and the distance function between two bodies (which is a nonlinear function of the configuration), sacrificing accuracy for speed. Depending on whether the distance function is approximated, the time-stepping schemes can also be divided into geometrically explicit schemes~\cite{AnitescuCP96, StewartT96} and geometrically implicit schemes~\cite{Tzitzouris01}. 

In geometrically explicit schemes, at the current state, a collision detection routine is called to determine separation or penetration distances between the  bodies, but this information is not incorporated as a function of the unknown future state at the end of the current time step. A goal of a typical time-stepping scheme is to guarantee consistency of the dynamic equations and all model constraints at the end of each time step. However, since the geometric information is obtained and approximated only at the start of the current time-step, then the solution will be in error. Apart from being geometrically explicit, most of the existing complementarity-based dynamic simulation methods and software also assume point contact between objects~\cite{CouBullet,SmithODE,TodorovET2012,TasoraS+2015,LeeG+2018,BerardT+2007}. A patch contact is usually approximated by ad hoc choice of $3$ contact points on the contact patch. In~\cite{XieC16}, we compared our non-point contact model with two popular point-based models, namely, Open Dynamic Engine (ODE)~\cite{SmithODE} and Bullet~\cite{CouBullet} in a pure translation task with a square contact patch where the analytic closed-form solution is known. We showed that our results matched the theoretical results, and was more accurate compared to ODE and Bullet. Thus, in~\cite{XieC16,NilanjanChakraborty2007}, we used a geometrically implicit time stepping scheme for solving convex contact patches problem, which is also the method used in this paper. The resulting discrete time problem is a MNCP.

\section{DYNAMIC MODEL FOR RIGID BODY SYSTEMS}
In complementarity methods, the dynamic simulation of intermittent unilateral contact between two rigid objects can be modeled by a geometrically implicit optimization-based time-stepping scheme. Note that the contact between objects is a planar contact patch, which can be either convex or non-convex. The dynamic model is made up of the following parts: (a) Newton-Euler equations (b) kinematic map relating the generalized velocities to the linear and angular velocities (c) friction law and (d) non-penetration constraints.  
The parts (a), (b) form a system of ordinary differential equations and they are standard for any complementarity-based formulation. Part (c) can be written as a system of complementarity constraints, which is based on Coulomb friction law using the maximum work dissipation principle. Part (d) incorporates the geometry of contact set as  system of complementarity constraint~\cite{NilanjanChakraborty2007,XieC16,XieC18a}.

To describe the dynamic model mathematically, we will introduce some notation first. Let $\bm{q}$ be the position of the center of mass of the object and the orientation of the object ($\bf{q}$ can be $6 \times 1$ or $7\times 1$ vector depending on the representation of the orientation). We will use unit quaternion to represent the orientation unless otherwise stated. The generalized velocity $\bm{\nu}$ is the concatenated vector of linear ($\bm{v}$) and spatial angular ($^s\bm{\omega}$) velocities. The effect of the contact patch is modeled as point contact of equivalent contact points (ECPs) $\bm{a}_1$ or $\bm{a}_2$ on two objects. Let
$\lambda_n$ be the magnitude of normal contact force,
$\lambda_t$ and $\lambda_o$ be the orthogonal components of the friction force on the tangential plane, and
$\lambda_r$ be the frictional moment about the contact normal.
\subsection{Newton-Euler equations of motion}
\begin{equation} 
\begin{aligned}
\label{eq1}
\bm{M}(\bm{q})
{\dot{\bm{\nu}}} &= 
\bm{W}_{n}\lambda_{n}+
\bm{W}_{t} \lambda_{t}+
\bm{W}_{o} \lambda_{o}
\\&+\bm{W}_{r}\lambda_{r}+
\bm{\lambda}_{app}+\bm{\lambda}_{vp}
\end{aligned}
\end{equation}
where $\bm{M}(\bm{q}) = \left[ \begin{matrix} &m \bm{I}_3 \ &0\\ &0 \ &{^s\mathcal{I}}_{cm}\end{matrix}\right]$ is a symmetric, positive definite $6 \times 6$ matrix, which contains mass matrix $m \bm{I}_3$ ($\bm{I}_3$ is a $3 \times 3$ identity matrix) and  inertia matrix ${^s\mathcal{I}}_{cm} = \bm{R} \mathcal{I}_{cm} \bm{R}^T$. Here $\bm{R}$ is the $3 \times 3$ rotation matrix from body frame to world frame and $\mathcal{I}_{cm}$ is the inertia matrix in the body frame. $\bm{\lambda}_{app}$ is the  $6 \times 1$ vector of external forces (including gravity) and moments, $\bm{\lambda}_{vp}$ is the $6 \times 1$ vector of Coriolis and centripetal forces. $\bm{W}_{n}$, $\bm{W}_{t}$, $\bm{W}_{o}$ and $
\bm {W}_{r}$ are dependent on configuration $\bm{q}$ and ECP ($\bm{a}_{1}$ or $\bm{a}_{2}$), and map the normal contact forces, frictional forces and moments to the world reference frame:
\begin{equation}
\begin{aligned}
\label{equation:wrenches}
\bm{W}_{n} =  \left [ \begin{matrix} 
\bm{n}\\
\bm{r}\times \bm{n}
\end{matrix}\right]
\quad 
\bm{W}_{t} =  \left [ \begin{matrix} 
\bm{t}\\
\bm{r}\times \bm{t}
\end{matrix}\right]
\\
\bm{W}_{o} =  \left [ \begin{matrix} 
\bm{o}\\
\bm{r}\times \bm{o}
\end{matrix}\right]
\quad 
\bm{W}_{r} =  \left [ \begin{matrix} 
\ \bm{0}\\
\ \bm{n} 
\end{matrix}\right]
\end{aligned}
\end{equation}
where $(\bm{n},\bm{t},\bm{o})$ are unit vectors of contact frame and $\bm{r}$ is the vector from center of gravity to the ECP, in the world frame.
\subsection{Kinematic map}
\begin{equation}
\bm{\dot{q}} = \bm{G}(\bm{q}) \bm{\nu}
\end{equation}
where matrix $\bm{G}$ maps the generalized velocity $\bm{\nu}$ to the time derivative of the position and orientation $\bm{\dot{q}}$.
\section{Modeling Planar Non-convex Patch Contact}
In this section, we will present our method for modeling a planar non-convex contact patch. Although, we will present the equations here in a more general manner, for concreteness, one can think  that one object is a non-convex object and the other object is a plane (or a face of a polyhedron). This is the scenario where planar non-convex contact patch is easy to visualize and this situation is quite prevalent in robotics.
Let $F$ and $G$ be the two objects, where, without loss of generality, the object $F$ is the non-convex object.
When two objects $F$ and $G$ have planar contact, the planar contact patch $\mathcal{S}$ is a non-empty finite subset of line or  plane. We will use the convex hull of object $F$, denoted by $Conv(F)$ to model the non-convex object $F$ (this will be justified later in the section). We will now present the contact constraints for non-penetration of rigid bodies.

\subsection{Non-penetration constraints}
In complementarity-based formulation of dynamics, the contact constraint for a potential contact is written as 
\begin{equation} \label{equation:normal contact}
0\le \lambda_{n} \perp \psi_{n}(
\bm{q},t) \ge 0
\end{equation}
where $\psi_{n}(\bm{q},t)$ is the gap function for the contact with the property $\psi_{n}(\bm{q},t) > 0$ for separation, $\psi_{n}(\bm{q},t) = 0$ for touching and $\psi_{n}(\bm{q},t) < 0$ for interpenetration. Note that there is usually no closed form expression for $\psi_{n}(\bm{q},t)$. Thus, usually, a call is made to a collision detection module that provides information on the distance function and a first order approximation of the above equation is usually used in the discrete-time formulation of equations of motion, which can lead to inaccuracies in motion prediction~\cite{NilanjanChakraborty2007}.

In~\cite{NilanjanChakraborty2007}, we presented a method for incorporating the geometry of the contacting objects so that, we make sure that Equation~\eqref{equation:normal contact} is satisfied exactly at the end of the time step and the contact points at the end of the time step are obtained. In~\cite{XieC16}, we showed that this method actually computes the ECP when the contact patch is a convex contact patch. We will now show that the contact constraints presented below allows us to compute the ECP of a non-convex contact patch as part of the integration of the equations of motion.

We assume that the convex hull of $F$, i.e., $Conv(F)$, and $G$ are
described by the intersecting convex inequalities $f_{i}(\bm{x}) \le 0, i = 1,...,m$, and $g_{j}(\bm{x}) \le 0, j = m+1,...,n$ respectively. Note that each individual convex constraint $f_{i}(\bm{x}) = 0$  describes the boundary of the convex hull. Note that multi-point contact, single point contact and convex patch contact are all special cases of the contact that we are considering. Let $\bm{a}_1$ and $\bm{a}_2$ be the pair of equivalent contact points for $Conv(F)$ and $G$ respectively. Note that, in general, $\bm{a}_1$ may not be a point in $F$.

We rewrite the contact condition (Equation~\eqref{equation:normal contact}) as a complementarity condition, and combine it with an optimization problem to find the closest points. Note that when objects are separate, the equivalent contact points $\bm{a}_{1}$ and $\bm{a}_{2}$ are solved as pair of closest points on the convex hull of $F$ and $G$. However, this does not lead to any inaccuracies since separation of $Conv(F)$ from the plane $G$ implies separation of $F$ and $G$ and vice-versa.  When objects have contact, $\bm{a}_{1}$ and $\bm{a}_{2}$ are solved as touching solution which prevents penetration between objects.

The convex inequality has the property that for any point $\bm{x}$, the point lies inside the object when $f(\bm{x}) < 0$, on the boundary of object when $f(\bm{x}) = 0$, and outside the object when $f(\bm{x}) > 0$. Thus, the contact condition (Equation~\eqref{equation:normal contact}) can be rewritten as either one of the following two complementarity constraints~\cite{NilanjanChakraborty2007}:
\begin{align}
\label{equation:contact_multiple_comp_1}
&0 \le \lambda_{n} \perp \mathop{max}_{i=1,...,m} f_{i}(\bm{a}_{2}) \ge 0\\
\label{equation:contact_multiple_comp_2}
&0 \le \lambda_{n} \perp \mathop{max}_{j=m+1,...,n}g_{j}(\bm{a}_{1}) \ge 0
\end{align}
where $\bm{a}_{1}$ and $\bm{a}_{2}$ are given by a solution to the following minimization problem:
\begin{equation}
\label{equation:optimazation}
(\bm{a}_{1},\bm{a}_{2}) = \arg \min_{\bm{\zeta}_1,\bm{\zeta}_2}\{ \|\bm{\zeta}_1-\bm{\zeta}_2 \| : \ f_{i}(\bm{\zeta}_1) \le 0,\ g_{j}(\bm{\zeta}_2) \le 0 \}
\end{equation}

As shown in~\cite{NilanjanChakraborty2007}, based on a modification of the KKT conditions, we can show that the ECPs need to satisfy the algebraic and complementarity constraints given below to solve the optimization problem above (Equation~\eqref{equation:optimazation}).
\begin{align}
\label{equation:re_contact_multiple_1}
&\bm{a}_{1}-\bm{a}_{2} = -l_{k}\nabla\mathcal{C}(\bm{F}_i,\bm{a}_{1})\\
\label{equation:re_contact_multiple_2}
&\nabla\mathcal{C}(\bm{F}_i,\bm{a}_{1})= -\sum_{j = m+1}^{n} l_{j} \nabla g_{j} (\bm{a}_{2})\\
\label{equation:re_contact_multiple_3}
&0 \le l_{i} \perp -f_{i}(\bm{a}_{1}) \ge 0 \quad i = 1,..,m,\\
\label{equation:re_contact_multiple_4}
&0 \le l_{j} \perp -g_{j}(\bm{a}_{2}) \ge 0 \quad j = m+1,...,n.
\end{align}
where $\nabla\mathcal{C}(\bm{F}_i,\bm{a}_{1}) = \nabla f_{k}(\bm{a}_{1})+\sum_{i\neq k}^{m} l_{i}\nabla f_{i}(\bm{a}_{1})$, $k$ represents the index of any one of the active constraints (i.e., the surface on which the ECP $\bm{a}_{1}, \bm{a}_{2}$ lies). We will also need an additional complementarity constraint (either Equation~\eqref{equation:contact_multiple_comp_1} or Equation~\eqref{equation:contact_multiple_comp_2}) to prevent penetration:
\begin{align}
\label{equation:re_contact_multiple_5}
0 \le \lambda_{n} \perp \mathop{max}_{j=m+1,...,n} g_{j}(\bm{a}_{1}) \ge 0
\end{align}
Equations~\eqref{equation:re_contact_multiple_1}$\sim$~\eqref{equation:re_contact_multiple_5} together gives the constraints that the equivalent contact points $\bm{a}_{1}$ and $\bm{a}_{2}$ for should satisfy for ensuring no penetration between the objects. We prove this formally in Proposition $2$. We first prove that the use of the convex hull ensures that the ECP that we compute is within the convex hull of the contact patch. 
\begin{proposition}
By using the convex hull of the object $F$ to formulate the contact constraints, we ensure that we compute the ECP within the convex hull of the contact patch.
\end{proposition}
\begin{proof}
Due to lack of space, we present a sketch of the proof idea here.
The convex hull contains the set of all the extreme points\footnote{The extreme point of a set is a point satisfying the following property: There exists a hyperplane passing through the point such that all points in the set lies on one side of the hyperplane.} of the object.  For a non-convex object contacting with a plane, the set of extreme points are the only points that can potentially contact the plane. Therefore, using the convex hull description ensures that we are capturing the set of all points that can be in contact. So when we are solving for the ECP, it will be in the convex region defined by the active constraints which is essentially the convex hull of the set of contacting points. 
\end{proof}

\begin{proposition}{When using Equations~\eqref{equation:re_contact_multiple_1} to~\eqref{equation:re_contact_multiple_5} to model the contact between convex hulls for two objects, we get the solution for ECPs as the closest points on the boundary of convex hulls respectively when objects are separate. When objects have planar contact, we will get touching solution which prevents penetration. }
\end{proposition}
\begin{proof}
Because of lack of space, we do not provide the full proof here. The proof essentially follows from the arguments of the proof shown in~\cite{NilanjanChakraborty2007} and~\cite{XieC16}, with minor modification to consider the convex hull of $F$ instead of $F$.
\end{proof}

\subsection{Friction Model}
Our friction model is based on the maximum power dissipation principle and generalized Coulomb's friction law.

The effect of the patch can be modeled as point contact based on the ECP $\bm{a}_1$ or $\bm{a}_2$:
\begin{equation}
\begin{aligned}
{\rm max} \quad -(v_t \lambda_t + v_o\lambda_o + v_r \lambda_r)\\
{\rm s.t.} \quad \left(\frac{\lambda_t}{e_t}\right)^2 + \left(\frac{\lambda_o}{e_o}\right)^2+\left(\frac{\lambda_r}{e_r}\right)^2 - \mu^2 \lambda_n^2 \le 0
\end{aligned}
\end{equation}
where $v_t$ and $v_o$ are the tangential components of the relative velocity at ECP of the contact patch, $v_r$ is the relative angular velocity about the normal at ECP. $e_t,e_o$ and $e_r$ is the given positive constants defining the friction ellipsoid and $\mu$ represents the coefficient of friction at the contact \cite{Howe1996, Trinkle1997}. This constraint is the elliptic dry friction condition suggested in \cite{Howe1996} based upon evidence from a series of contact experiments. This model states that among all the possible contact forces and moments that lie within the friction ellipsoid, the forces and moment that maximize the power dissipation at the contact (due to friction) are selected.

This argmax formulation of the friction law has a useful alternative formulation~\cite{trinkle2001dynamic}
 \begin{equation}
\begin{aligned}
\label{equation:friction}
0&=
e^{2}_{t}\mu \lambda_{n} 
\bm{W}^{T}_{t}\cdot
\bm{\nu}+
\lambda_{t}\sigma\\
0&=
e^{2}_{o}\mu \lambda_{n}  
\bm{W}^{T}_{o}\cdot
\bm{\nu}+\lambda_{o}\sigma\\
0&=
e^{2}_{r}\mu \lambda_{n}\bm{W}^{T}_{r}\cdot
\bm{\nu}+\lambda_{r}\sigma\\
\end{aligned}
\end{equation}
\begin{equation}
\label{equation:friction_c}
0 \le \mu^2\lambda_{n}^2- \lambda_{t}^2/e^{2}_{t}- \lambda_{o}^2/e^{2}_{o}- \lambda_{r}^2/e^{2}_{r} \perp \sigma \ge 0
\end{equation}
where $\bm{W}^{T}_{(.)}$ are dependent on ECP for the contact patch and $\sigma$ is the magnitude of the slip velocity on the contact patch.

\subsection{Time-stepping Formulation}
We use a velocity-level formulation and an Euler time-stepping scheme to discretize the above system of equations. Let $t_u$ denote the current time and $h$ be the duration of the time step, the superscript $u$ represents the beginning of the current time and the superscript $u+1$ represents the end of the current time. Using $\dot{\bm{\nu}} \approx ( {\bm{\nu}}^{u+1} -{\bm{\nu}}^{u} )/h$, $\dot{\bm{q}} \approx( {\bm{q}}^{u+1} -{\bm{q}}^{u} )/h$ and writing forces as impulses, we get the discretized Newton-Euler equations and kinematic map:
\begin{align}
\label{equ:discrete_NE}
&0 = 
-\bm{M}^{u+1}({\bm{\nu}}^{u+1} - {\bm{\nu}}^{u})+ \bm{P}^{u+1}_c+\bm{p}^{u}_{app}+\bm{p}^{u}_{vp}\\
\label{equ:discrete_KE}
&0 =-\bm{q}^{u+1}+\bm{q}^{u}+h\bm{G}(\bm{q}^u)\bm{\nu}^{u+1}
\end{align}
where impulse $p_{(.)}=h\lambda_{(.)}$, the contact impulse $\bm{P}^{u+1}_c$ is:
\begin{equation}
\bm{P}^{u+1}_c = \bm{W}_{n}p^{u+1}+\bm{W}_{t}p^{u+1} +\bm{W}_{o}p^{u+1}_{o}+\bm{W}_{r}p^{u+1}_{r}
\end{equation}
where $\bm{W}_{t}, \bm{W}_{n}, \bm{W}_{o}, \bm{W}_{r}$ are dependent on ECPs at the end of time step $u+1$.

We discretize contact constraints (Equations~\eqref{equation:re_contact_multiple_1}$\sim$~\eqref{equation:re_contact_multiple_5}) and friction model (Equations~\eqref{equation:friction} and~\eqref{equation:friction_c}) by writing forces $\lambda_{(.)}$ into impulses $p_{(.)}$. Furthermore, the unknown contact impulses in Equations~\eqref{equation:friction}, ~\eqref{equation:friction_c} and unknown ECPs in Equations~\eqref{equation:re_contact_multiple_1}$\sim$~\eqref{equation:re_contact_multiple_5} are at the end of time step $u+1$.

\subsection{Summary of geometrically implicit time-stepping scheme}
As stated earlier, our dynamic model is composed of (a) Newton-Euler equations (Equation~\eqref{equ:discrete_NE}), (b) kinematic map between the rigid body generalized velocity and the rate of change of the parameters for representing position and orientation (Equation~\eqref{equ:discrete_KE}), (c) contact model  which gives the constraints that the equivalent contact points $\bm{a}_{1}$ and $\bm{a}_{2}$ should satisfy for ensuring no penetration between the objects (Equations~\eqref{equation:re_contact_multiple_1}$\sim$~\eqref{equation:re_contact_multiple_5}). (d) friction model  which gives the constraints that contact wrenches should satisfy (Equations~\eqref{equation:friction} and~\eqref{equation:friction_c}). Thus, we have a coupled system of algebraic and complementarity equations (mixed nonlinear complementarity problem) that we have to solve.

\begin{figure*}%
\begin{subfigure}{0.7\columnwidth}
\includegraphics[width=\columnwidth]{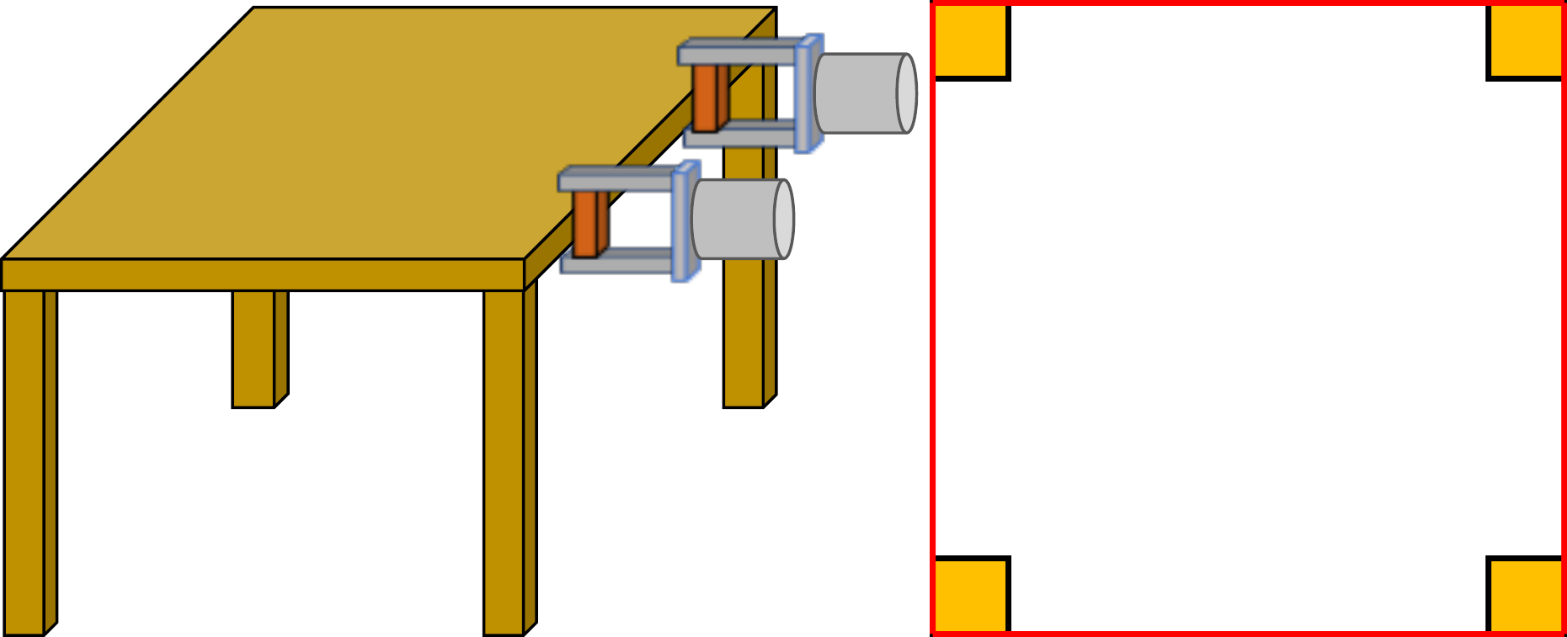}%
\caption{A four-legged desk on ground is pushed by two grippers (Left), where the contact between desk feet and ground is a union of four squares (Right). We get the convex hull for the contact patch (red square).}
\label{figure:ex1_1} 
\end{subfigure}\hfill
\begin{subfigure}{0.4\columnwidth}
\includegraphics[width=\columnwidth]{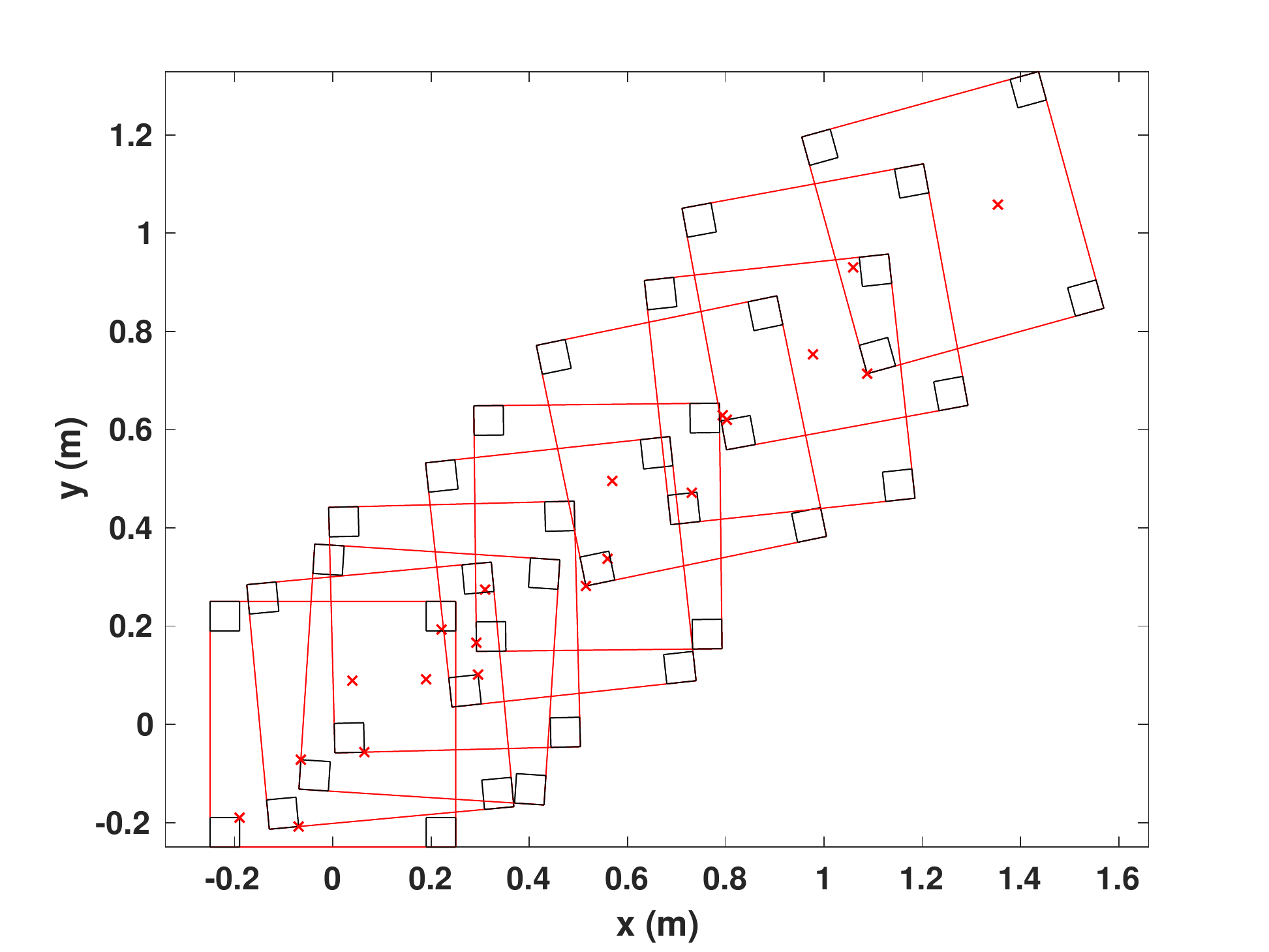}%
\caption{The snapshot for the contact patch between desk feet and ground during the motion. The ECP is shown in red dot.}
\label{figure:ex1_2} 
\end{subfigure}\hfill%
\begin{subfigure}{0.9\columnwidth}
\includegraphics[width=\columnwidth]{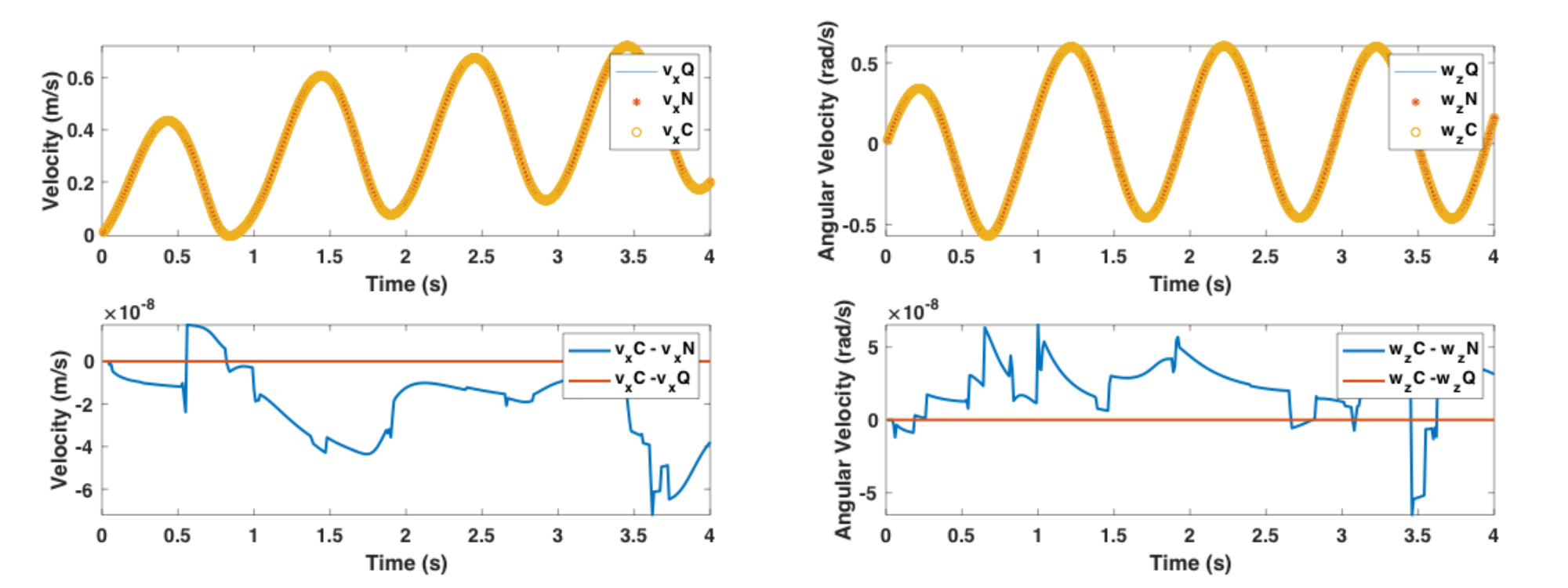}%
\caption{(Top Left) The solution of translational velocity from~\cite{XieC18b} ($v_xQ$), from~\cite{XieC18a} ($v_xN$) and proposed method ($v_xC$). (Bottom Left) The difference between $v_xC$ and $v_xN$, and difference between $v_xC$ and $v_xQ$. (Top Right) The angular velocity $w_zQ$, $w_zN$ and $w_zC$, (Bottom Right) and the differences between them.}
\label{figure:ex1_3} 
\end{subfigure}%
\caption{Comparison of the proposed method with~\cite{XieC18a,XieC18b}.  }
\label{Example1}
\end{figure*}

\begin{figure*}%
\begin{subfigure}{1\columnwidth}
\includegraphics[width=\columnwidth]{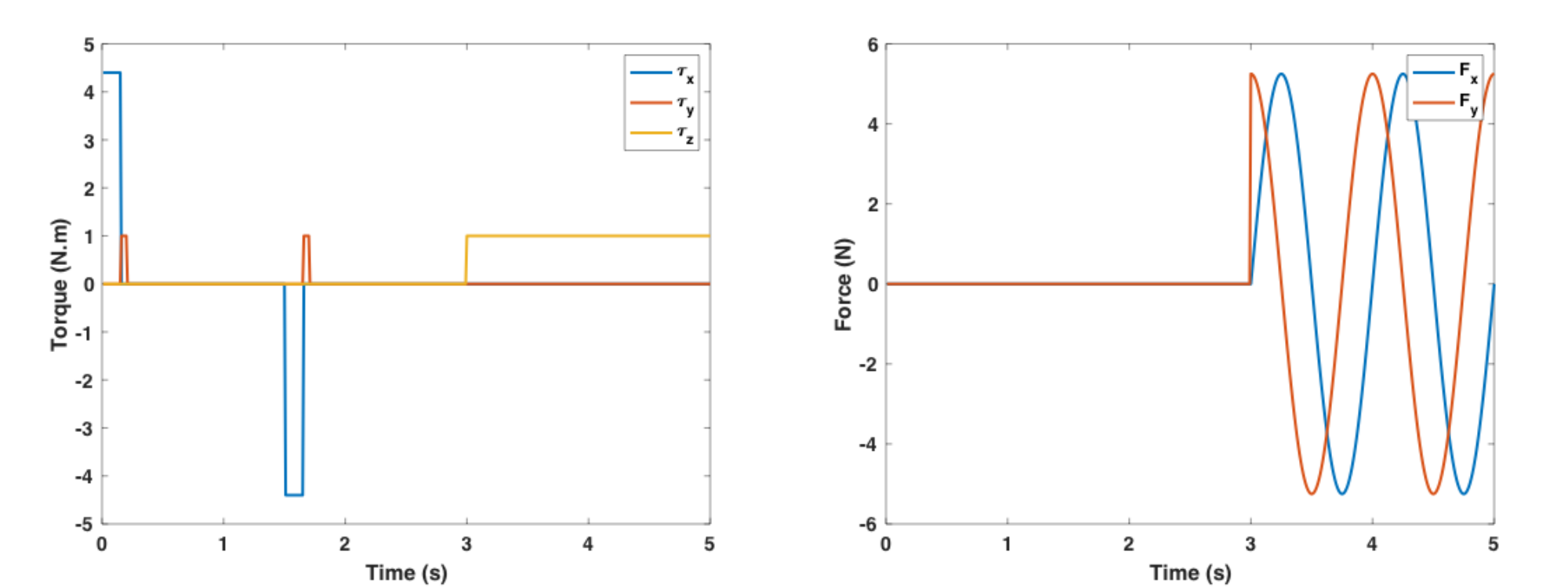}%
\caption{(Left) We plot the applied torques ($\tau_x,\tau_y,\tau_z$) from the gripper exerting on the T-shaped bar along with the time. (Right) In addition, we plot the applied force ($F_x,F_y$) acting on the bar.}
\label{figure:ex2_2} 
\end{subfigure}\hfill%
\begin{subfigure}{1\columnwidth}
\includegraphics[width=\columnwidth]{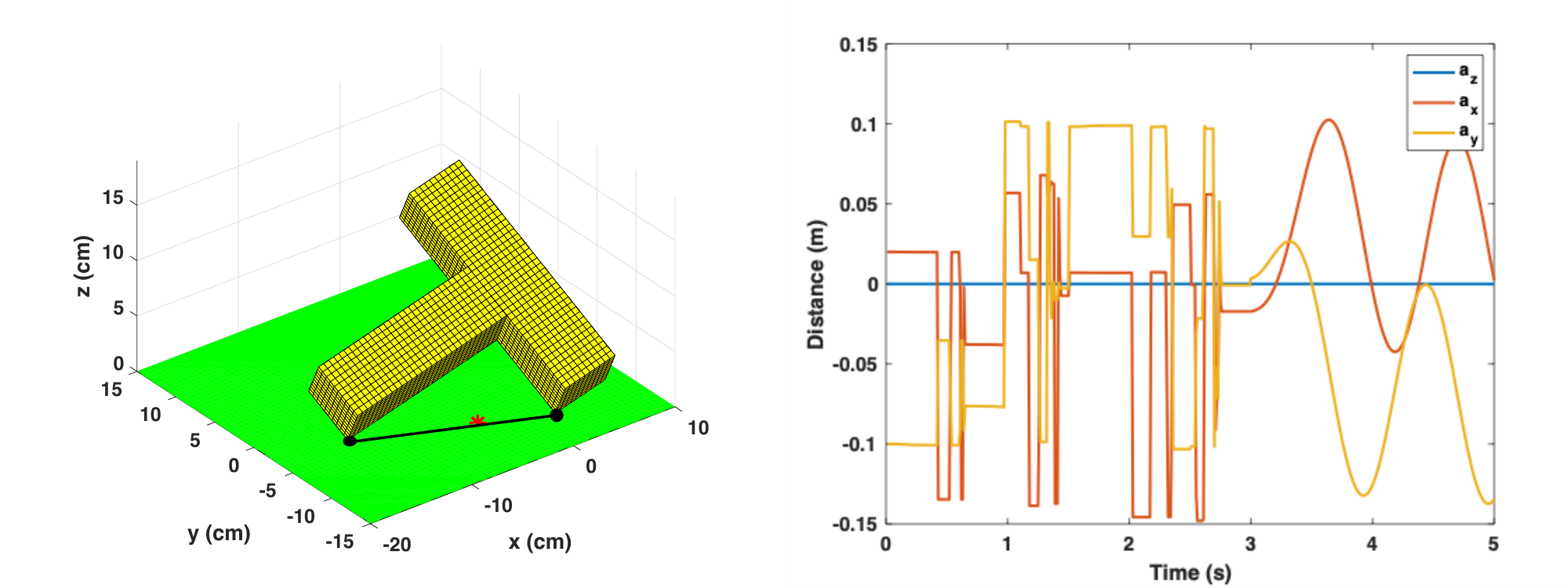}%
\caption{(Left) The two-point contact which is non-convex between the bar and ground is replaced by the convex line contact in red. (Right) The coordinates of ECP ($a_x,a_y,a_z$) during the motion.}
\label{figure:ex2_1} 
\end{subfigure}
\begin{subfigure}{0.2\textwidth}
\includegraphics[width=\textwidth]{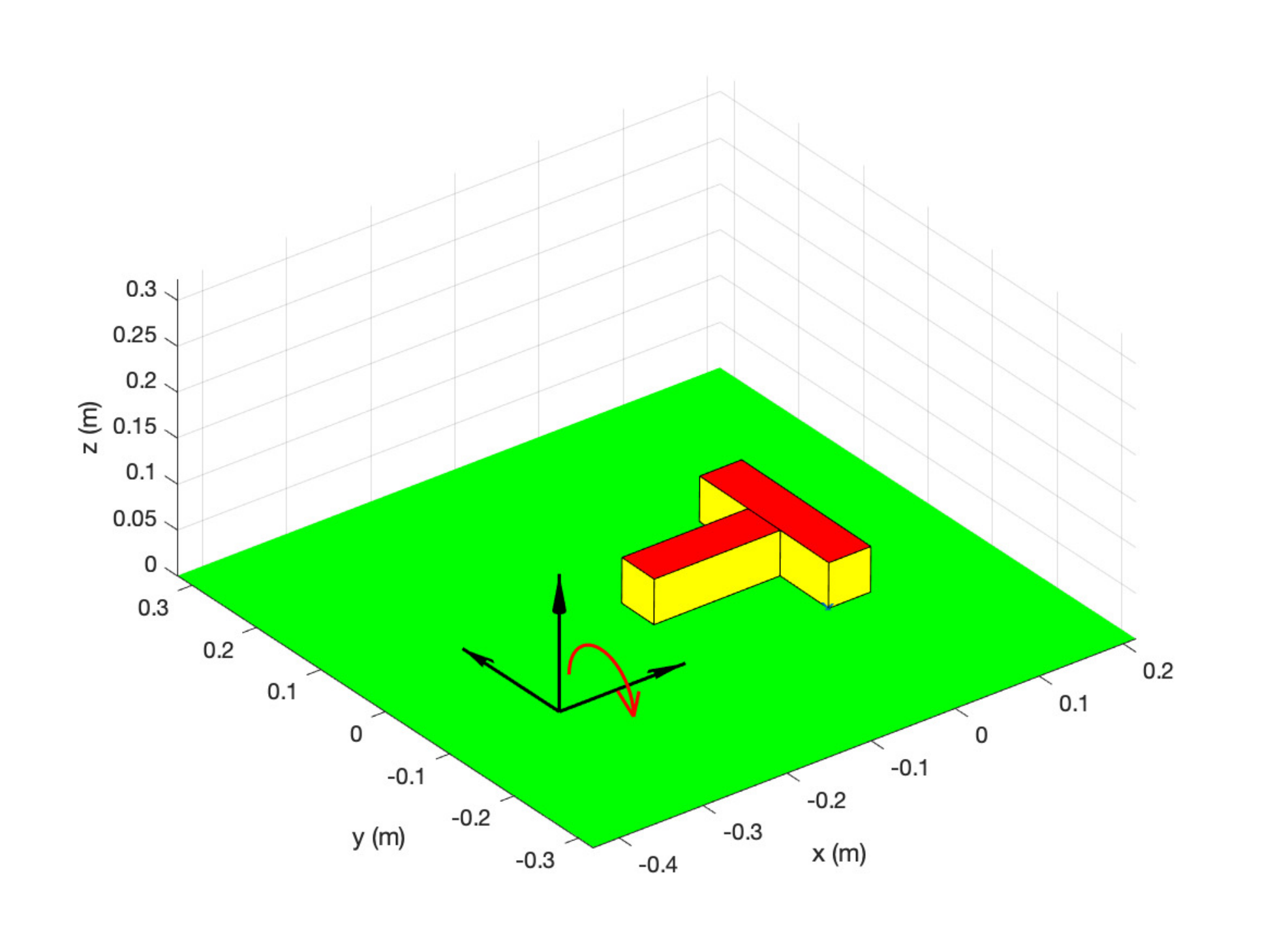}%
\caption{The snapshot: t = 0.01s  }
\label{figure:ex2_3} 
\end{subfigure}\hfill%
\begin{subfigure}{0.2\textwidth}
\includegraphics[width=\textwidth]{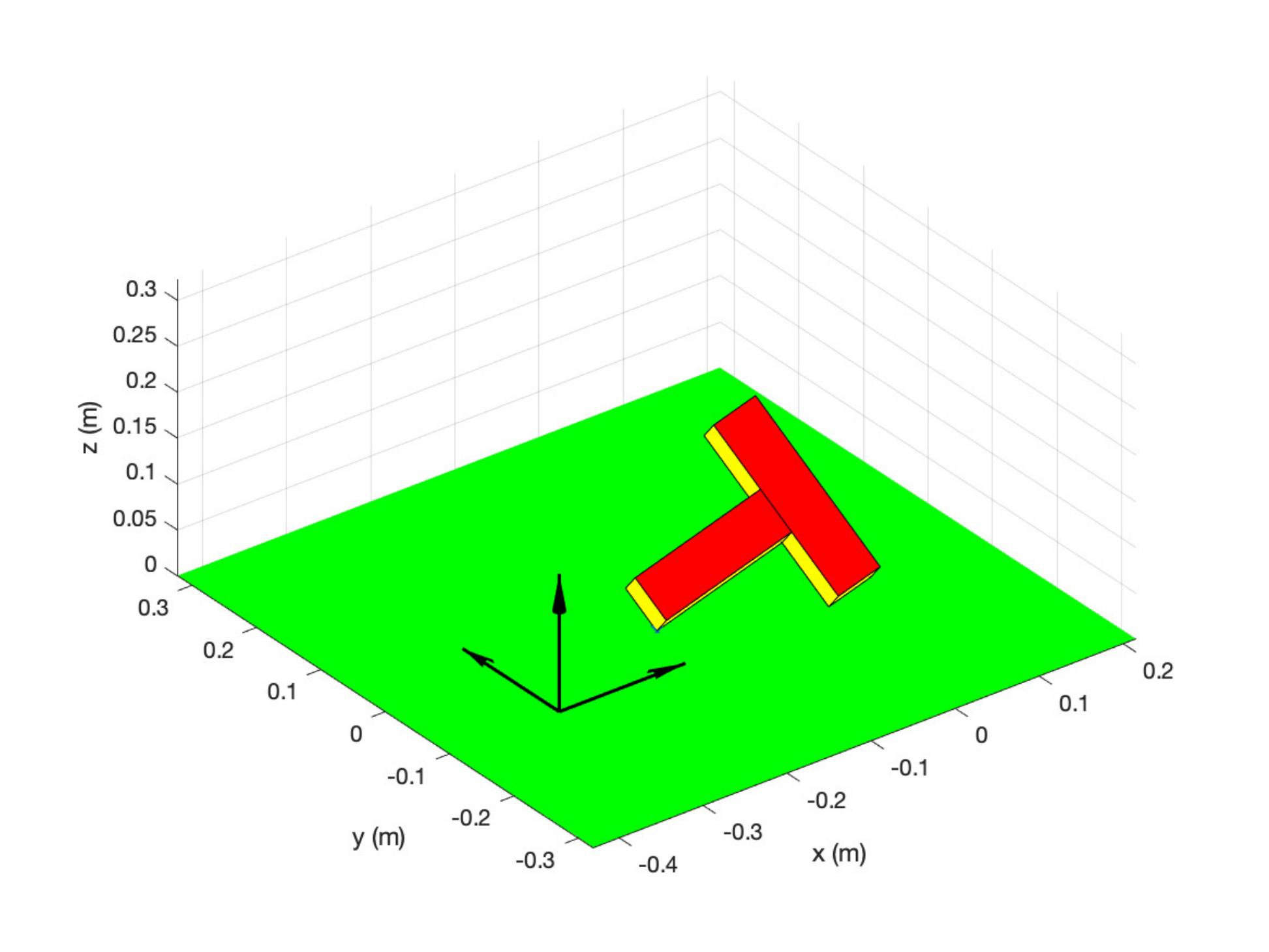}%
\caption{The snapshot: t = 0.51s }
\label{figure:ex2_4} 
\end{subfigure}\hfill%
\begin{subfigure}{0.2\textwidth}
\includegraphics[width=\textwidth]{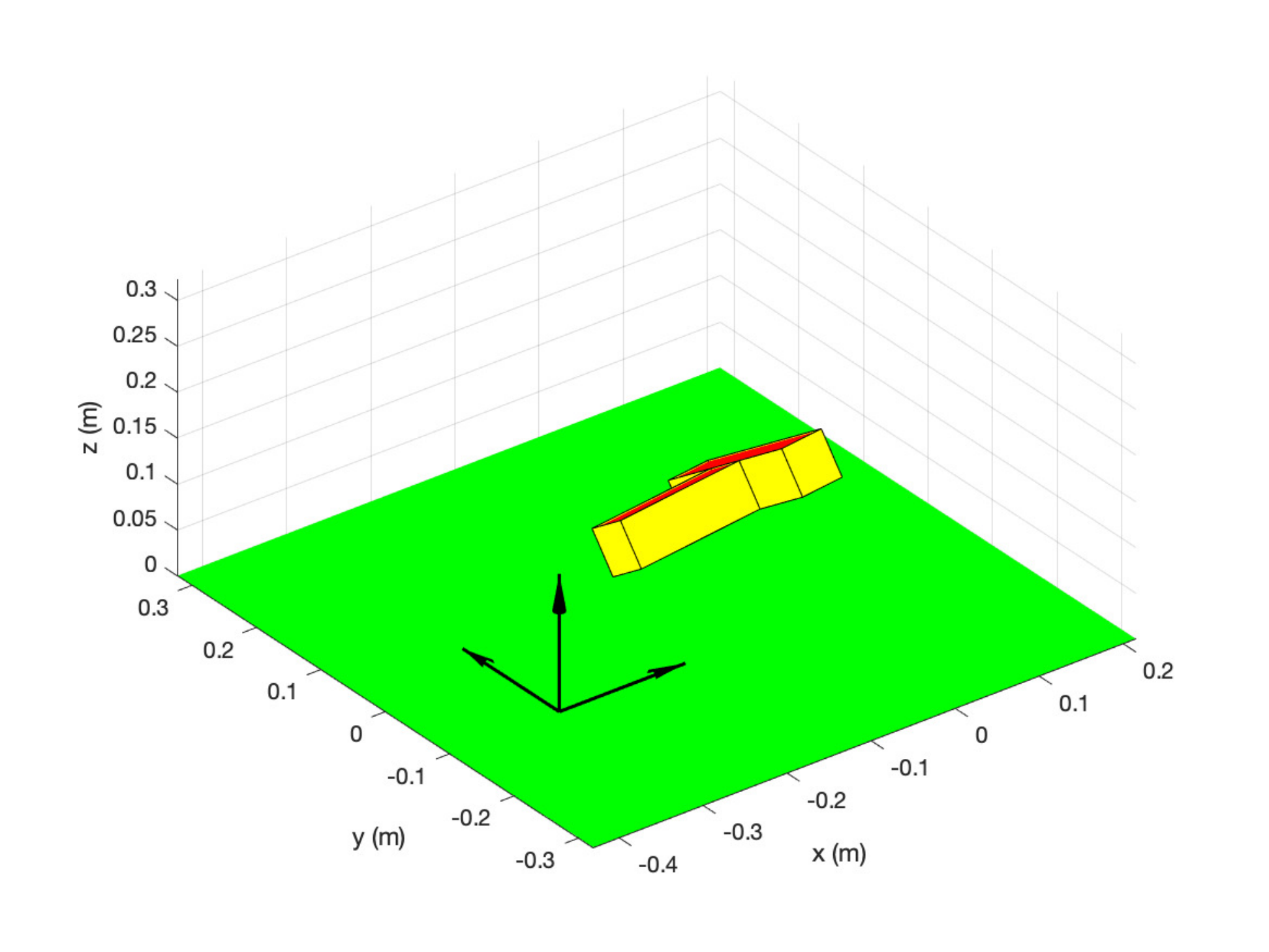}%
\caption{The snapshot: t = 1.76s }
\label{figure:ex2_5} 
\end{subfigure}\hfill%
\begin{subfigure}{0.2\textwidth}
\includegraphics[width=\textwidth]{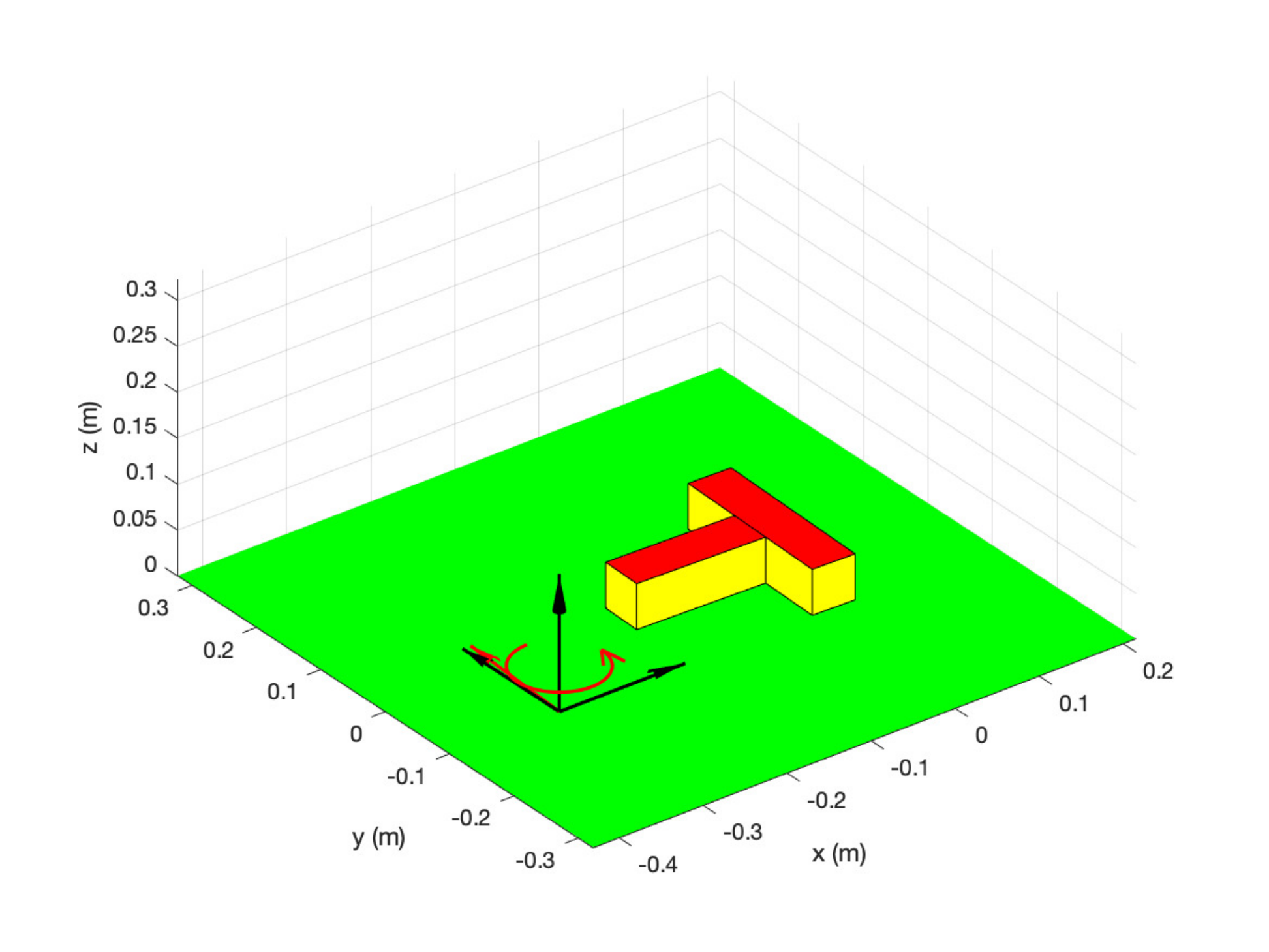}%
\caption{The snapshot: t = 3.01s }
\label{figure:ex2_6} 
\end{subfigure}\hfill%
\begin{subfigure}{0.2\textwidth}
\includegraphics[width=\textwidth]{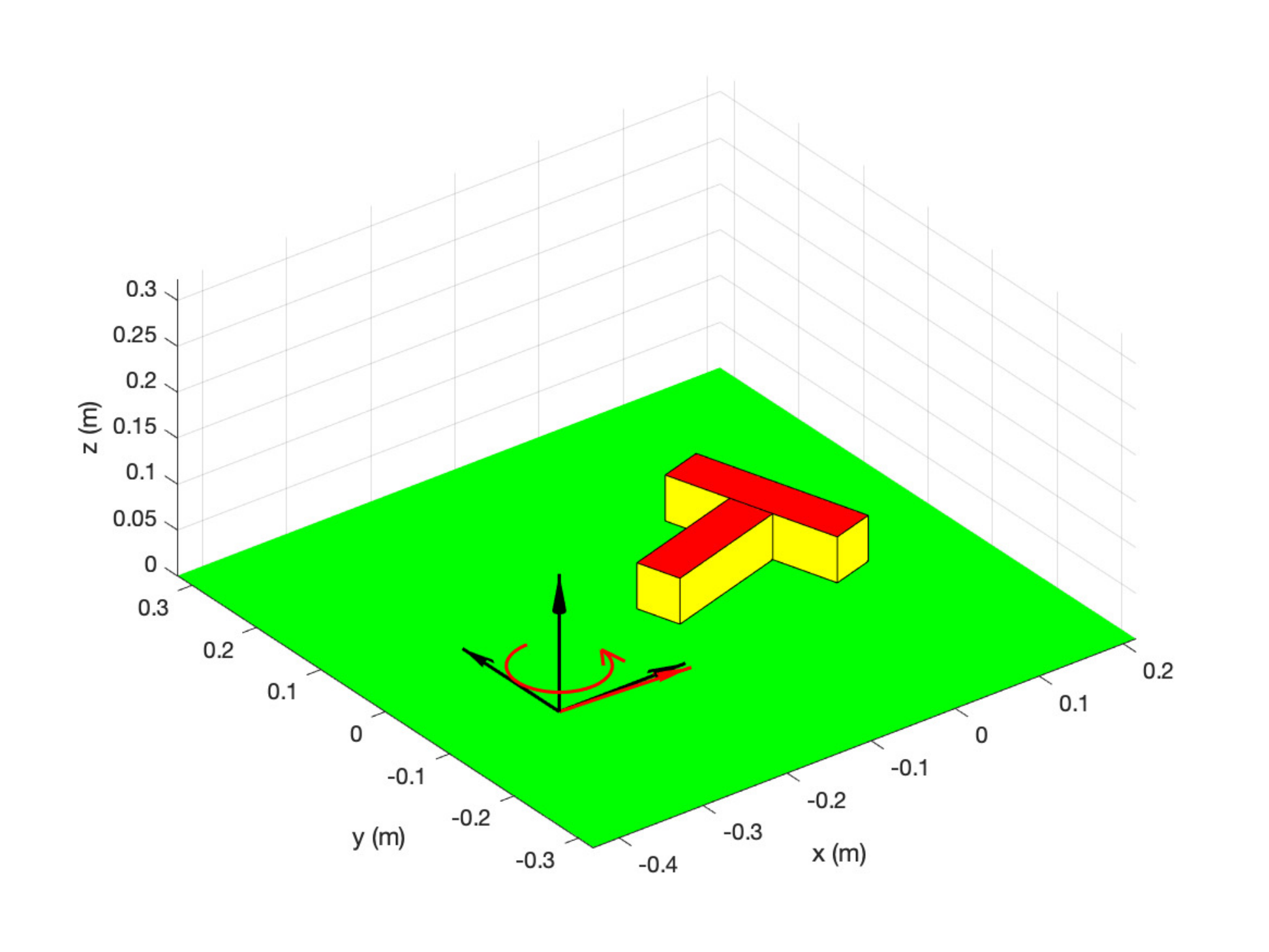}%
\caption{The snapshot: t = 3.26s }
\label{figure:ex2_7} 
\end{subfigure}\hfill%
\caption{Simulation for the motion of T-shaped bar example based on the proposed method.  }
\label{Example2}
\end{figure*}

\section{SIMULATION RESULTS}
In this section, we evaluate the performance of our proposed method on two example problems. All the simulations are run in MATLAB on a MacBook Pro with 2.6 GHZ processor and 16 GB RAM.  
\subsection{Comparison with existing methods}
We first consider the problem of predicting motion of a square desk with four legs pushed by two grippers, where the contact patch between desk feet and support is a union of four small squares (see Figure~\ref{figure:ex1_1}). Such problems are useful for robots navigating or rearranging furniture in domestic environments. The dimension for the square desk is length $L = 0.5m$, the length for each small square is $L_s = 0.06m$ and height of desk's CM is $H = 0.45m$. The mass of desk is $m = 15kg$ and the gravity's acceleration is $g = 9.8m/s^2$.

We compare the convex hull-based contact detection method presented in this paper with method in~\cite{XieC18a} and~\cite{XieC18b}. Note that we apply forces, so as to ensure sliding without toppling. Therefore, the dynamic model for sliding motion that we proposed in~\cite{XieC18b} can be applied here. In~\cite{XieC18b}, we have shown that for sliding-only motion, the discrete-time equations of motion can be reduced to a system of four quadratic equations. Since the contact is a union of four disjoint squares, we can also use the method in~\cite{XieC18a}, where, we consider each non-penetration constraint between each contact patch and the ground separately.   

The time step chosen for all the simulations is $h = 0.01s$ and simulation time is $4s$. The coefficient of friction between desk and support is $\mu=0.22$ and the given constants for friction ellipsoid are $e_t = e_o = 1, e_r = 0.1m$. As shown in Figure~\ref{figure:ex1_1}, the desk slides on the support. The initial position of CM is $q_x = q_y = 0m$, $q_z = 0.45m$ and orientation about normal axis is $\theta_z = 0$. The initial velocity is $v_x = 0.3m/s$, $v_y = 0.2m/s$, $w_z = 0.5 rad/s$. The external forces and moments from grippers exerted on the desk is periodic, $f_x = 22.5\sin(2\pi t)+ 22.5 \ N$,$f_y = 22.5\cos(2\pi t) + 22.5 \ N$, $\tau_z = 2.1\cos(2\pi t) \ Nm$, where $t \in [0,4]s$.

In Figure~\ref{figure:ex1_2}, we plot the snapshot for the contact patch during the motion. It can be seen from the figure, that the table translates as well as rotates during motion. The ECP is marked by a red cross and it can be seen that the ECP is not within the contact patch and it is also not below the center of mass of the table (which matches the intuition, since the table is rotating).  In Figure~\ref{figure:ex1_3}, we plot the velocity of the desk ($v_x$ and $w_z$) during the motion. In addition, we plot the difference between solutions of quadratic model and convex hull method, and difference between solutions of MNCP model and convex hull method. As Figure~\ref{figure:ex1_3} illustrates, the differences for $v_x$ and $w_z$ are within 1e-8, which validates the accuracy of convex hull method.

Furthermore, the average time the model in~\cite{XieC18b} spends for each time step is $0.0022s$, the time our proposed method method spends is $0.0053s$ (which is $2.4$ times than~\cite{XieC18b}), and the time the model in~\cite{XieC18a} spends is $0.0487s$ (which is more than $22$ times than quadratic model's and $9$ times than convex hull method). To summarize, proposed method simplify the model in~\cite{XieC18a} greatly by modeling multiple contact patches with a single patch and therefore is much more efficient. The model in~\cite{XieC18b}, although faster is valid only for sliding and cannot be applied to situations where the object may topple.
\subsection{Simulations of the T-shaped bar}
This example is used to illustrate that our method allows objects to automatically transition between different contact modes (surface, point, line and also making and breaking of contact), while ensuring the objects do not penetrate. As Figure~\ref{figure_Motivation} illustrates, the planar contact patch between T-shaped bar and the support is non-convex. 

The dimensions of the bar are given in Figure~\ref{figure_Motivation}. The mass of the bar is $2kg$, the other parameters like gravity and friction parameters are the same as in the first example. The time step chosen is $h = 0.01s$ and the total simulation time is $t = 5s$. Figure~\ref{figure:ex2_2} shows the applied forces and moments from the gripper acting on the bar, and Figure~\ref{figure:ex2_1} demonstrates the coordinates of ECP (i.e., $a_x,a_y,a_z$). Note that the coordinate of ECP along z axis $a_z$ stays zero within the numerical tolerance of $1e^{-12}$ during the motion. Thus, there is no penetration between the bar and ground. The snapshots show the transition of the bar from surface contact~\ref{figure:ex2_3} to two-point contact~\ref{figure:ex2_4} to another two-point contact with different pair of contact points~\ref{figure:ex2_5} to a surface contact~\ref{figure:ex2_6} and then rotation while having the surface contact. All these transitions were automatically detected by our algorithm.

\section{Conclusion}
In this paper we present a geometrically implicit time-stepping method for solving dynamic simulation problems with planar non-convex contact patches. In our model, we use a convex hull of the non-convex object and combine the collision detection with numerical integration of equations of motion. This allows us to solve for an equivalent contact point (ECP) in the convex hull of the non-convex contact patch as well as the contact wrenches simultaneously. We prove that although we model the contact patch with an ECP, the non-penetration constraints at the end of the time-step are always satisfied. We present numerical simulation motion prediction for two example problems that are representative of applications in robotic manipulation. The results demonstrate  that our method can automatically transition among different contact modes (non-convex contact patch, point, and line). 






\bibliographystyle{IEEEtran}
\bibliography{main}


\end{document}